\documentclass[journal]{IEEEtran}

\usepackage{epsfig}
\usepackage{graphicx}
\usepackage{amsmath}
\usepackage{amsfonts}
\usepackage{amssymb}
\usepackage{semtrans}
\usepackage{color}
\usepackage{algpseudocode}
\usepackage{algorithm}

\usepackage{hyperref}
\definecolor{darkred}{RGB}{150,0,0}
\definecolor{darkgreen}{RGB}{0,150,0}
\definecolor{darkblue}{RGB}{0,0,200}
\hypersetup{colorlinks=true, linkcolor=darkred, citecolor=darkgreen, urlcolor=darkblue}

\newtheorem{theorem}{Theorem}[section]
\newtheorem{lemma}{Lemma}[section]

\newtheorem{definition}{Definition}[section]

\newtheorem{remark}[subsection]{Remark}

\newtheorem{example}[subsection]{Example}

\newcommand{\onenorm}[1]{\left\|#1\right\|_{\ell_1}}
\newcommand{\twonorm}[1]{\left\|#1\right\|_{\ell_2}}

\newcommand{\infnorm}[1]{\left\|#1\right\|_{\ell_\infty}}

\newcommand{\abs}[1]{\left|#1\right|}

\newcommand{\A}{\boldsymbol{A}}

\newcommand{\R}{\mathbb{R}}
\renewcommand{\Re}{\mathbb{Re}}
\renewcommand{\S}{\mathcal{S}}
\newcommand{\0}{\boldsymbol{0}}
\renewcommand{\a}{\boldsymbol{a}}
\renewcommand{\b}{\boldsymbol{b}}

\newcommand{\e}{\boldsymbol{e}}
\renewcommand{\c}{\boldsymbol{c}}

\newcommand{\x}{\boldsymbol{x}}
\newcommand{\y}{\boldsymbol{y}}
\renewcommand{\v}{\boldsymbol{v}}
\newcommand{\ty}{\tilde{\boldsymbol{y}}}
\newcommand{\z}{\boldsymbol{z}}
\newcommand{\X}{\boldsymbol{X}}
\newcommand{\Y}{\boldsymbol{Y}}
\newcommand{\Z}{\boldsymbol{Z}}
\newcommand{\N}{\mathcal{N}}

\newcommand{\bG}{\boldsymbol{\Gamma}}
\newcommand{\bW}{\mathbb{W}}

\newcommand{\st}{\operatorname{s.t.}}
\newcommand{\argmin}{\operatorname{argmin}}

\newcommand{\ie}{\text{i.e.}}
\newcommand{\eg}{\text{e.g.}}

\numberwithin{equation}{section} 
\def \endprf{\hfill {\vrule height6pt width6pt depth0pt}\medskip}
\newenvironment{proof}{\noindent {\bf Proof} }{\endprf\par}

\ifCLASSINFOpdf
\else
\fi

\begin{document}
\title{Approximate Subspace-Sparse Recovery with Corrupted Data via Constrained $\ell_1$-Minimization}

\author{Ehsan~Elhamifar,~\IEEEmembership{Member,~IEEE,}
        Mahdi~Soltanolkotabi,~\IEEEmembership{Member,~IEEE,}
        and~S.~Shankar~Sastry,~\IEEEmembership{Fellow,~IEEE}
\IEEEcompsocitemizethanks{\IEEEcompsocthanksitem E. Elhamifar is an assistant professor in the College of Computer and Information Science, Northeastern University, USA. E-mail: eelhami@ccs.neu.edu.%
\IEEEcompsocthanksitem M. Soltanolkotabi is an assistant professor in the Department
of Electrical Engineering, University of Southern California, USA. E-mail: soltanol@usc.edu.
\IEEEcompsocthanksitem S. Shankar Sastry is a professor of the Electrical Engineering and Computer Sciences Department, University of California at Berkeley, USA. E-mail: sastry@eecs.berkeley.edu.}
\thanks{}}

\maketitle

\begin{abstract}
High-dimensional data often lie in low-dimensional subspaces corresponding to different classes they belong to. Finding sparse representations of data points in a dictionary built using the collection of data helps to uncover low-dimensional subspaces and address problems such as clustering, classification, subset selection and more. In this paper, we address the problem of recovering sparse representations for noisy data points in a dictionary whose columns correspond to corrupted data lying close to a union of subspaces. We consider a constrained $\ell_1$-minimization and study conditions under which the solution of the proposed optimization satisfies the approximate subspace-sparse recovery condition. More specifically, we show that each noisy data point, perturbed from a subspace by a noise of the magnitude of $\varepsilon$, will be reconstructed using data points from the same subspace with a small error of the order of $O(\varepsilon)$ and that the coefficients corresponding to data points in other subspaces will be sufficiently small, \ie, of the order of $O(\varepsilon)$. We do not impose any randomness assumption on the arrangement of subspaces or distribution of data points in each subspace. Our framework is based on a novel generalization of the null-space property to the setting where data lie in multiple subspaces, the number of data points in each subspace exceeds the dimension of the subspace, and all data points are corrupted by noise. Moreover, assuming a random distribution for data points, we further show that coefficients from the desired support not only reconstruct a given point with high accuracy, but also have sufficiently large values, \ie, of the order of $O(1)$.
\end{abstract}

\begin{IEEEkeywords}
Low-dimensional subspaces, sparse representation, noisy data points, $\ell_1$-minimization, subspace incoherence, subspace inradius, approximate recovery.
\end{IEEEkeywords}

\IEEEpeerreviewmaketitle

\section{Introduction}
\IEEEPARstart{H}{igh-dimensional} datasets are ubiquitous in many areas of science, such as computer vision, information retrieval, image processing, bio and health informatics and more. Real-world data, however, often lie close to low-dimensional subspaces instead of being uniformly distributed in the high-dimensional ambient space \cite{Basri:PAMI03, Tomasi:IJCV92, Hastie:StatSci98, Hong:TIP06, Yang:CVIU08, Elhamifar:TPAMI13}. Exploiting and recovering low-dimensional structures in data, in fact, is the key to efficiently address a variety of important problems such as classification \cite{Wright:PAMI09, Elhamifar:TSP12}, clustering \cite{Elhamifar:TPAMI13, Elhamifar:CVPR09, Lerman:Annals11, Lerman:CA14, Favaro:CVPR11, DelaTorre:CVPR12, Heckel:ISIT13}, 
subset selection \cite{Elhamifar:CVPR12, Esser:TIP12, Elhamifar:NIPS12}, visualization as well as other applications \cite{Eldar:TIT09, Montanari:UAI12, Lerman:ICCV13, Sapiro:ICLR14}. 

Sparse representation techniques provide effective tools to exploit and uncover the low-dimensional structures in datasets \cite{Donoho:CPAM06, Candes-Tao:TIT05, Tibshirani:RSS96}. 
More specifically, given a measurement $\y \in \Re^n$ and a dictionary or a sensing matrix $\A \in \Re^{n \times N}$, which has a nontrivial null-space, the goal of sparse recovery is to find a representation $\c \in \Re^N$ of $\y$ as a linear combination of the columns of $\A$, such that $\c$ has only a few nonzero coefficients. A computationally efficient method to achieve this goal is to solve the $\ell_1$-minimization program
\begin{equation}
\min \| \c \|_1 ~~ \st ~~ \y = \A \c.
\end{equation}
In fact, $\| \c \|_1$, which is the sum of the absolute values of elements of $\c$, is the convex envelope of the cardinality of $\c$ and is known to recover sparse solutions, under appropriate conditions on the dictionary and the sparsity level \cite{Donoho:CPAM06, Candes-Tao:TIT05, Tibshirani:RSS96, Elad:SIAM09, Yu:AnnalStat09}. 

Sparse representation-based methods can be divided into two categories, depending on the type of dictionaries being used. The first group of methods uses fixed pre-defined dictionaries, such as the ones built from Wavelets, Fourier basis, Random Projections and so on \cite{Elad:SIAM09, Candes:TIT06, Candes:SPM08}. The second group of methods uses adaptive dictionaries built from the collection of data, where the columns of the dictionary $\A$ correspond to data points \cite{Elhamifar:TPAMI13, Wright:PAMI09, Elhamifar:Annals14}. In fact, the latter has achieved or outperformed state-of-the-art results in clustering and classification of high-dimensional data.
%
Under the assumption that the data points lie in a union of subspaces, with the number of data in each subspace being larger than the dimension of the subspace, a sparse representation of $\y$, ideally, corresponds to a subspace-sparse representation. In other words, $\y$ can be written as a linear combination of a few data points that lie in the same low-dimensional subspace. In fact, subspace-sparse recovery is the key requirement for the success of sparse representation-based clustering, classification and subset selection algorithms \cite{Elhamifar:TPAMI13, Elhamifar:TSP12, Heckel:ISIT13, Elhamifar:CVPR12, Elhamifar:ICASSP10, Candes:ASTAT12, Leng:NIPS13, Elhamifar:TPAMI14}. 
One can show that when data points perfectly lie in subspaces, under appropriate conditions on the affinities between subspaces and the distribution of data, the solution of $\ell_1$-minimization perfectly recovers a subspace-sparse representation \cite{Elhamifar:TPAMI13, Candes:ASTAT12}.

An important challenge related to real-world datasets is that data points are often corrupted by noise. In other words, not only $\y$, but also all columns of the dictionary $\A$ are corrupted by noise. As a result, standard analysis tools related to the first group of sparse recovery methods, in which the predefined dictionary $\A$ is uncorrupted while the measurement $\y$ is noisy, are not applicable \cite{Eldar:TIT09, Donoho:CPAM06-2, Candes:RIP08}. 
Recently, \cite{Elhamifar:Annals14, Wang:ICML13} studied the problem of subspace-sparse recovery in the presence of noise using the unconstrained optimization program
\begin{equation}
\label{eq:lasso0}
\min \lambda \| \c \|_1 + \frac{1}{2} \twonorm{\y - \A \c}^2,
\end{equation}
where the regularization parameter $\lambda > 0$ sets a trade-off between sparsity and reconstruction error objectives. \cite{Elhamifar:Annals14} shows that, when data points are drawn uniformly at randomly from the intersection of the hypersphere and subspaces, under appropriate conditions on subspace affinities and data points and for certain range of $\lambda$, the solution of \eqref{eq:lasso0} recovers subspace-sparse representations for all data points. On the other hand, \cite{Wang:ICML13} analyzes the solution of \eqref{eq:lasso0} under more general settings, including the deterministic case where both subspaces and data points are fixed, and under appropriate conditions, proves exact subspace-sparse recovery.

Notice that while the performance of both unconstrained and constrained $\ell_1$ has been analyzed in conventional sparse recovery \cite{Yu:AnnalStat09, Donoho:CPAM06-2, Candes-Romberg-Tao:CPAM06, Tropp:TIT06}, where $\A$ is noise-free without the multi-subspace structure redundancy, the analysis of sparse recovery for the case of noisy multi-subspace data with corrupted dictionary has been limited to unconstrained $\ell_1$-minimization in \eqref{eq:lasso0}. In fact, we believe that this is partly due to the fact that analyzing the constrained optimization is much harder and requires development of new analysis tools. Moreover, current results \cite{Wang:ICML13, Elhamifar:Annals14} do not show how large coefficients from the desired support can be, which is an important factor for successful clustering and classification \cite{Elhamifar:TPAMI13, Wright:PAMI09}.

\medskip\noindent\textbf{Paper Contributions.} 
In this paper, we study the problem of approximate subspace-sparse recovery in the presence of noise using the constrained $\ell_1$-minimization program
\begin{equation}
\label{eq:L0}
\min \| \c \|_1 ~~ \st ~~ \twonorm{\y - \A \c} \leq \gamma \varepsilon,
\end{equation}
with a regularization parameter $\gamma > 0$, which we determine in the paper. We consider the general settings, where we do not impose any randomness assumption on the arrangement of subspaces or distribution of data points in each subspace. We assume that all data points are corrupted by Gaussian noise whose Euclidean norm is smaller than or equal to $\varepsilon$. We show that, under appropriate conditions on the data and subspaces, the solution of \eqref{eq:L0} satisfies the approximate subspace-sparse recovery property, \ie, 1) $\y$ will be reconstructed using data point from its underlying subspace with an error that is of the order of   $O(\varepsilon)$; 2) coefficients corresponding to data points in other subspaces are sufficiently small, of the order~of~$O(\varepsilon)$. Our theoretical results relies on a novel generalization of the well-known null-space property, studied in conventional sparse recovery \cite{DonohoElad:PNAS03,Gribonval:TIT03, Stojnic:TSP09, VandenBerg:TIT10}, to the setting where 1) data lie in a union of subspaces, with the number of data points in each subspace typically being larger than the subspace dimension; 2) all data points are corrupted by~noise. Moreover, assuming random distribution for data points, we further show that in the solution of \eqref{eq:L0}, coefficients from the desired support not only reconstruct $\y$ with high accuracy, but also have sufficiently large values, \ie, are of the order of $O(1)$.

\medskip\noindent\textbf{Paper Organization.} The organization of this paper is as follows. In Section \ref{sec:probstatement}, we present the settings of our problem. We state the approximate subspace-sparse recovery problem and introduce appropriate definitions and notations. In Section \ref{sec:theory}, we present our theoretical guarantees for the constrained $\ell_1$-minimization program. Finally,  Section \ref{sec:conclusion} concludes the paper.

\section{Problem Formulation and Main Results}
\label{sec:probstatement}
In this section, we consider the problem of finding sparse representations for corrupted data points that lie close to a union of subspaces. Assume that we have $L$ linear subspaces $\{ \S_i \}_{i=1}^{L}$ in $\R^n$ of dimensions $\{ d_i \}_{i=1}^{L}$. Let $\X \in \R^{n \times N}$ denote a matrix whose columns correspond to noise-free data points that lie in the union of the $L$ subspaces. Without loss of generality, we assume that the columns of $\X$ have unit Euclidean norms. We denote by $\X_i \in \R^{n \times N_i}$ the $N_i$ data points that lie in $\S_i$, hence $\sum_{i=1}^{L}{N_i} = N$. We can write 
\begin{equation}
\label{eq:Xdef}
\X \triangleq \begin{bmatrix} \X_1 & \X_2 & \cdots & \X_L \end{bmatrix} \boldsymbol{\Gamma} \in \R^{n \times N},
\end{equation}
where $\bG \in \Re^{N \times N}$ is a permutation matrix, which is not necessarily known a priori. Given $\x$ that lies in one of the subspaces, the subspace-sparse recovery problem refers to the problem of finding a representation of $\x$ in the dictionary $\X$, as $\x = \X \c$, such that the nonzero coefficients of $\c$ correspond to a few data points that lie in the same subspace as that of $\x$. More specifically, considering the sparse optimization program
\begin{equation}
\c^* = \arg\min \onenorm{\c} \quad \st \quad \x = \X \c,
\end{equation}
one would like to have a few nonzero elements in $\c^*$ that correspond to data points lying in the same subspace of $\x$. 

In real-world problems, however, data points often do not lie perfectly in subspaces, due to corruption by noise. Instead, they lie approximately close to a union of subspaces. In this paper, we address the problem of approximate subspace-sparse recovery in the presence of noise. More precisely, we assume that we have a collection of noisy data points $\Y_i \in \R^{n \times N_i}$ from each subspace $\S_i$, \ie,
\begin{equation}
\Y_i = \X_i + \Z_i,
\end{equation}
where $\X_i$ denotes the collection of noise-free data points, which lie at the intersection of $\S_i$ with the unit hypersphere, and $\Z_i$ denotes the random noise matrix, which has i.i.d elements drawn from the Gaussian distribution $\N(0,\frac{\epsilon^2}{n})$. As a result, each noise-free data point of unit Euclidean norm on each subspace is corrupted by a noise whose Euclidean norm is roughly less than or equal to $\varepsilon$, where
\begin{equation}
\varepsilon \triangleq \epsilon (1 + \rho),
\end{equation}
for a sufficiently small $\rho > 0$. We also assume that $\x \in \S_i$, which has unit Euclidean norm, is corrupted by a noise $\z$, which has i.i.d elements drawn from $\N(0,\frac{\epsilon^2}{n})$, giving rise to the noisy data point $\y = \x + \z$. 
\vspace{1mm}
\begin{remark}
Notice that for a Gaussian random vector $\z \in\R^n$ with i.i.d entries drawn from $\N(0,\frac{\epsilon^2}{n})$, with high probability, we have $\twonorm{\z} \leq \varepsilon$. For the sake of brevity, throughout the paper, we do not include explicitly the failure probability of $\twonorm{\z} \leq \varepsilon$ in the probabilistic statements of our~results. 
\end{remark}
\vspace{1mm}
For simplicity of notation, we denote $\X = \begin{bmatrix} \X_i & \X_{-i}  \end{bmatrix}$, where $\X_{-i}$ represents the collection of data points from all subspaces except $\S_i$. Similarly, we write $\Y = \begin{bmatrix} \Y_i & \Y_{-i}  \end{bmatrix}$, where $\Y_{-i}$ denotes the collection of noisy data points from all subspaces except $\S_i$. We also use the convention $\y \in \S_i^{\varepsilon}$ to refer to a noisy data point that is the sum of a noise-free data point $\x$ in $\S_i$ with unit Euclidean norm and a noise $\z$ whose Euclidean norm is smaller than or equal to $\varepsilon$, \ie, 
\begin{equation}
\label{eq:noisySi}
\S_i^\varepsilon \triangleq \{ \y \in \Re^n: ~\y = \x + \z, ~\x \in \S_i, ~\twonorm{\z} \leq \varepsilon \}.
\end{equation}
Our goal is to find an approximate subspace-sparse representation, $\c^\top = \begin{bmatrix} \c_i^\top & \c_{-i}^\top \end{bmatrix}$, of a noisy data point $\y$ in the dictionary of corrupted data, $\Y$, as we define next. 
\vspace{1mm}
\begin{definition} [approximate subspace-sparse recovery]
\label{def:approx-subspace-sparse-recovery}
Consider a noisy data point $\y$ lying in $\S_i^\varepsilon$ and a noisy dictionary $\Y$, where the Euclidean norm of the noise on the its columns is less than or equal to $\varepsilon.$ An approximate subspace-sparse recovery of $\y$ in $\Y$ corresponds to a representation $\y = \Y \c$, such that 

\begin{figure}[t!]
\centering
\includegraphics[width=0.482\linewidth, trim = 0 0 0 0, clip]{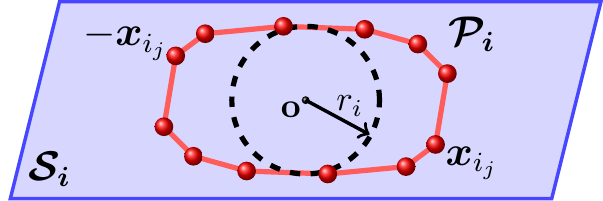}\ \hspace{1mm}
\includegraphics[width=0.482\linewidth, trim = 0 0 0 0, clip]{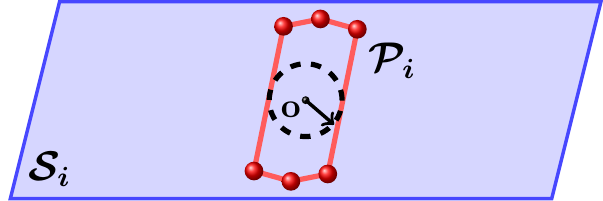}
\caption{Left: The subspace inradius associated with $\S_i$ is the radius of the largest Euclidean ball whose intersection with $\S_i$ is inscribed in the symmetrized convex hull of data points in $\S_i$. Right: When data are not well distributed in a subspace, i.e., they are close to a degenerate subspace, \eg, a line inside a plane, the subspace inradius decreases.}
\label{fig:def-inradius}
\end{figure}
\begin{figure}[t!]
\centering
\includegraphics[width=0.47\linewidth, trim = 0 0 0 0, clip]{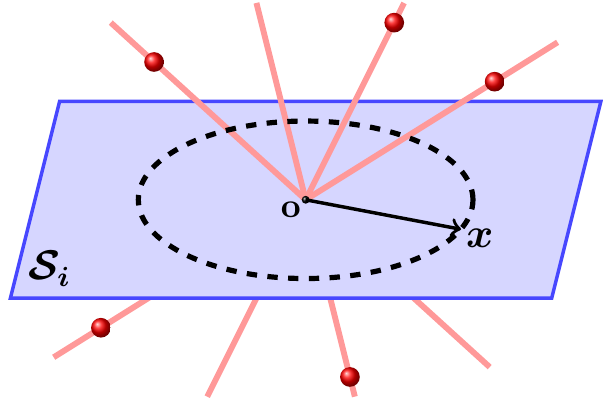}
\vspace{-2mm}
\caption{The subspace incoherence associated with $\S_i$ is defined as the maximum inner product between an arbitrary vector of unit Euclidean norm in $\S_i$ and data points in other~subspaces.}
\label{fig:def-incoherence}
\end{figure}

\noindent 1) $\y$ can be reconstructed, with high accuracy, using noisy data points from its own subspace, \ie,
\begin{equation}
\label{eq:approx1}
\twonorm{\y - \Y_i \c_i} \leq O(\varepsilon);
\end{equation}
2) the nonzero coefficients corresponding to noisy data points in other subspaces are sufficiently small, \ie,
\begin{equation}
\label{eq:approx2}
\onenorm{\c_{-i}} \leq O(\varepsilon).
\end{equation}
\end{definition}
In order to achieve an approximate subspace-sparse representation for $\y$ in the dictionary $\Y$, in this paper, we consider the the constrained $\ell_1$-minimization program 
\begin{equation}
\label{eq:L1noisy1}
\min \onenorm{\c} \quad \st \quad \twonorm{\y - \Y \c} \leq \gamma \varepsilon,
\end{equation}
where $\gamma > 0$ is a parameter that we determine in the paper. We investigate conditions on the data and subspaces under which the optimal solution of \eqref{eq:L1noisy1} achieves approximate subspace-sparse recovery. 

The conditions that we derive depend on the inradius of convex bodies of the data in each subspace and the incoherence between subspaces, as we define next.

\vspace{1mm}
\begin{definition}[subspace inradius] Let $\X_i \triangleq \begin{bmatrix} \x_{i1} & \x_{i2} & \cdots & \x_{iN_i} \end{bmatrix}$ be a matrix whose columns lie in $\S_i$. Denote by $P(\X_i)$ the symmetrized convex hull of $\X_i$. More precisely,
\begin{equation}
P(\X_i) \triangleq \operatorname{conv} (\pm \x_{i1} , \pm \x_{i2}, \ldots, \pm \x_{iN_i}).
\end{equation}
The subspace inradius associated with $\S_i$, which we denote by $r_i$, is defined as the radius of the largest Euclidean ball whose intersection with $\S_i$ is inscribed in $P(\X_i)$, see Figure \ref{fig:def-inradius}.
\end{definition}

\vspace{1mm}
\begin{definition}[subspace incoherence] 
\label{def:data-subspace-incoherence}
The subspace incoherence associated with $\S_i$ is defined as
\begin{equation}
\mu_i \triangleq \max_{\x \in \S_i,\twonorm{\x}=1}{\infnorm{ {\x^{\top} \X_{-i}} }}.
\end{equation}
In other words, $\mu_i$ is the maximum inner product between an arbitrary vector of unit Euclidean norm in $\S_i$ and the columns of $\X_{-i}$, which correspond to data points in other subspaces, see Figure \ref{fig:def-incoherence}.
\end{definition}
\vspace{1mm}
Notice that from the definition of the principal angles between subspaces, we always have $\mu_i \leq \max_{j \neq i} \cos \theta_{ij}$, where $\theta_{ij}$ denotes the smallest principal angle between $\S_i$ and $\S_j$.

In this paper, we show that, as long as the subspace incoherences between $\S_i$ and other subspaces are sufficiently small compared to the subspace inradius of $\S_i$, the optimization algorithm in \eqref{eq:L1noisy1}, for an appropriate $\gamma$, finds an approximate subspace-sparse representation for any $\y \in \S_i^\varepsilon$. More specifically, we prove the following result.
\begin{theorem}
\label{thm:maintheory}
\emph{
Let $\gamma \triangleq \max_{i}{2(1+\frac{2 \sqrt{2( \log N_i+\log n)}}{r_i})}$. Define $\beta$ as 
\begin{equation}
\label{eq:beta0}
\beta \triangleq ( 1 + \max_{i} \,  \frac{3 r_i}{r_i - (\mu_i + \varepsilon)} ) \, \frac{\gamma}{2} + \delta,
\end{equation}
where $\delta > 0$ is arbitrarily small.
Then, for every $i$ and every $\y \in \S_i^{\varepsilon}$, the solution of the optimization problem in \eqref{eq:L1noisy1}, denoted by $\c^{*\top} = \begin{bmatrix} \c_i^{*\top} & \c_{-i}^{*\top} \end{bmatrix}$, with high probability, satisfies 
\begin{equation}
\twonorm{\y - \Y_i \c_i^*} \leq \beta \varepsilon.
\end{equation}
In addition, assuming $\varepsilon \leq \frac{\gamma}{2 \beta + \gamma} r_i$, with high probability, we~have
\begin{equation}
\onenorm{\c_{-i}^*} \leq \frac{2 \beta+\gamma}{2 r_i} \varepsilon.
\end{equation}
}
\end{theorem}
Notice that, from \eqref{eq:beta0}, a necessary condition for the approximate subspace-sparse recovery is to have $\mu_i + \varepsilon < r_i$ for all $i$. This in fact makes sense since from earlier results \cite{Elhamifar:TPAMI13, Candes:ASTAT12}, in the noise-free setting, perfect subspace-sparse recovery holds as long as $\mu_i < r_i$ holds for all $i$. Thus, given the fact that data points are corrupted by noise whose Euclidean norm is about $\varepsilon$, there is a need for adjustment of the condition by incorporating the noise level. 

\begin{remark}
Our results in Theorem \ref{thm:maintheory} suggest that the smaller the ratio $(\mu_i + \varepsilon)/r_i$ is and the larger $r_i$ is, the better recovery we obtain using the $\ell_1$-minimization in \eqref{eq:L1noisy1}. This is expected, since a larger $r_i$ corresponds to a more even distribution of points in subspace $i$, i.e., farther from a degenerate subspace. On the other hand, a smaller $\mu_i$ corresponds to points in different subspaces being more dissimilar to each~other. 
\end{remark}
\begin{example}
Let $1/ \kappa \triangleq \max_{i} (\mu_i + \varepsilon)/r_i$, where $\kappa > 1$, from the necessary condition that $r_i$ must be greater than $\mu_i + \varepsilon$, as stated earlier. Also, let $c \triangleq 2 \sqrt{2 (\log N_{i^*}+\log n) }$ and $r \triangleq r_{i^*}$, where $i^*$ is the index for which we obtain the maximum value in the definition of $\gamma$ in Theorem \ref{thm:maintheory}. Hence, we have 
\begin{equation}
\gamma = 2 \, (1 + \frac{c}{r}).
\end{equation}
In addition, using the definition of $\beta$ in \eqref{eq:beta0}, we can write
\begin{equation}
\beta = (4 + \frac{3}{\kappa - 1}) (1 + \frac{c}{r}) + \delta.
\end{equation}
Clearly, the larger the value of subspace inradius $r$ is, the less error tolerance $\gamma$ we can allow for the reconstruction of a given $\y$ and, at the same time, the reconstruction of $\y$ using noisy points in its own subspace has a smaller error. In addition, as $\kappa$ increases, the error on the reconstruction of $\y$ using noisy points in its own subspace decreases. In the limiting case of $\kappa$ being large enough, we obtain
\begin{equation}
\twonorm{\y - \Y_i \c_i^*} \leq 4 \, (1 + \frac{c}{r}) \, \varepsilon.
\end{equation}
\end{example}

Assuming random distribution for data points, we can further show that in the solution of the $\ell_1$-minimization program  \eqref{eq:L1noisy1}, the coefficients from the correct support, \ie, $\c_i^*$, not only reconstruct $\y$ with a high accuracy, but also have sufficiently large values. More specifically, we show the following result.
\begin{theorem}
Assume that the noise-free data in each subspace $\S_i$, \ie, the columns of $\X_i$, are drawn uniformly at random from the intersection of the unit hypersphere with $\S_i$. Let $\c^{*\top} = \begin{bmatrix} \c_i^{*\top} & \c_{-i}^{*\top} \end{bmatrix}$ be the solution of the optimization program in \eqref{eq:L1noisy1} for a noisy data point $\y$ in $\S_i^{\varepsilon}$. Assume that the approximate reconstruction condition $\twonorm{\y - \Y_i \c_i^*} \leq \beta \varepsilon$ holds.
%
Then, with high probability, we have 
\begin{equation}
\label{eq:approxSuppRecov2}
\onenorm{\c_i^*} \geq \frac{1 - (\beta + 1) \, \varepsilon}{2 \sqrt{ \frac{2\log N_i}{d_i} } + 2 \sqrt{ \frac{2\log N_i}{n} } \, \varepsilon}.
\end{equation}
\end{theorem}

In the next section, we provide the required theoretical analysis tools to prove the above results. In fact, our theory relies on a novel generalization of the null-space property \cite{DonohoElad:PNAS03,Gribonval:TIT03, Stojnic:TSP09, VandenBerg:TIT10} to the setting where 1) data lie in a union of subspaces, with the number of data points in each subspace typically larger than the subspace dimension; 2) all data points are corrupted by noise.

\section{Approximate Subspace-Sparse Recovery Theory}
\label{sec:theory}

In this section, we consider the $\ell_1$-minimization program
\begin{equation}
\label{eq:L1-noisy}
\begin{split}
\begin{bmatrix} \c_i^* \\ \c_{-i}^* \end{bmatrix}  \, = \, &\argmin \onenorm{ \begin{bmatrix} \c_i \\ \c_{-i} \end{bmatrix} } \\ &\st \quad \twonorm{ \y - \begin{bmatrix} \Y_i & \Y_{-i} \end{bmatrix} \begin{bmatrix} \c_i \\ \c_{-i} \end{bmatrix} } \leq \gamma \varepsilon,
\end{split}
\end{equation}
and investigate conditions under which we achieve approximate subspace-sparse recovery for an arbitrary noisy data point $\y \in \S_i^{\varepsilon}$. More precisely, we investigate conditions under which the optimal solution of \eqref{eq:L1-noisy} approximately reconstructs $\y$ from noisy data points in its own subspace, \ie, $\twonorm{\y - \Y_i \c_i^*}$ is bounded by $O(\varepsilon)$, and the coefficients corresponding to noisy data points in other subspaces are sufficiently small, \ie, $\onenorm{\c_{-i}^*}$ is of the order of  $O(\varepsilon)$. 

\subsection{Preliminary Lemmas}
To prove the main results of the paper, we make use of the following Lemmas. The proof of the first Lemma can be found in \cite{Candes:ASTAT12} and we provide the proofs of the other two Lemmas in the Appendix. 
\vspace{1mm}
\begin{lemma}
\label{lem:L1boundnonoise}
\emph{
Given a noise-free data point $\x \in \S_i$, the $\ell_1$-norm of the optimal solution of the minimization program
\begin{equation}
\c_i^*=\argmin \onenorm{\c} \quad \st \quad \x = \X_i \c,
\end{equation}
satisfies the following inequality 
\begin{equation}
\onenorm{\c_i^*} \leq \frac{\twonorm{\x}}{r_i}.
\end{equation}
}
\end{lemma}
In other words, the upper bound on the minimum $\ell_1$-norm representation of a noise-free data point $\x$ in $\S_i$ in terms of noise-free data points in $\S_i$ is proportional to the Euclidean norm of $\x$ and is inversely proportional to the subspace inradius $r_i$.
\vspace{1mm}
\begin{lemma}
\label{lem:noisebound} 
\emph{
For $\Z_i \in \R^{n \times N_i}$ with i.i.d entries drawn from $\N(0,\frac{\epsilon^2}{n})$ and a given $\c_i \in \R^{N_i}$, with probability at least $1-\frac{1}{(n N_i)^2}$, we have
\begin{equation}
\label{eq:noisenormbound}
\twonorm{\Z_i \c_i}\le 2 \varepsilon \, \sqrt{2(\log N_i+\log n)} \, \onenorm{\c_i}.
\end{equation}
}
\end{lemma}
\vspace{1mm}
The result of the above Lemma implies that given $\Y_i = \X_i + \Z_i$ whose columns are noisy data points in $\S_i^{\varepsilon}$, the linear combination $\Y_i \c_i$ corresponds to perturbing the noise-free vector $\X_i \c_i$ lying in $\S_i$ with a noise whose Euclidean norm is bounded above by \eqref{eq:noisenormbound}.
\vspace{1mm}
\begin{lemma}
\label{lem:L1bound}
\emph{
Given a noisy data point in the $i$-th subspace, $\y \in \S_i^{\varepsilon}$, consider the $\ell_1$-minimization program
\begin{equation}
\label{eq:L1noisy}
\c^* = \argmin \onenorm{\c} \quad \st \quad \twonorm{\y - \Y \c} \leq \gamma \varepsilon,
\end{equation}
with $\gamma \triangleq \max_{i}{2(1 + \frac{2 \sqrt{2(\log N_i+\log n)}}{r_i})}$. With probability at least $1-\frac{1}{(n N_i)^2}$, we have
\begin{equation}
\onenorm{\c^*} \leq \frac{1}{r_i}.
\end{equation}
}
\end{lemma}
Thus, for an appropriately chosen error tolerance, the upper bound on the $\ell_1$-norm of the optimal representation of a noisy data point in $\S_i^{\varepsilon}$, as a linear combination of all noisy data points in $\Y$, is inversely proportional to the subspace inradius~$r_i$.
 
As a consequence of Lemmas \ref{lem:noisebound} and \ref{lem:L1bound}, for the optimal solution of \eqref{eq:L1-noisy}, we have
\begin{equation}
\begin{split}
\twonorm{\Z_i \c_i^*} &\leq 2 \varepsilon \, \sqrt{2(\log N_i+\log n)} \, \onenorm{\c_i^*} \\ &\leq 2 \varepsilon \, \frac{\sqrt{2(\log N_i+\log n)}}{r_i},
\end{split}
\end{equation}
where we used the fact that $\onenorm{\c_i^*} \leq \onenorm{\c^*} \leq \frac{1}{r_i}$.

\subsection{Main Results}
In this section, we prove our main result in Theorem \ref{thm:maintheory}. To do so, we consider an arbitrary vector $\tilde{\y}$ that lies close to $\S_i$ and whose Euclidean norm is larger than the approximate recovery noise level, \ie, $\twonorm{\tilde{\y}} > \beta \varepsilon$, where $\beta > 0.5 \gamma$. We consider the following $\ell_1$-minimization programs,
\begin{eqnarray}
\label{eq:subL1noisy1}
\!\!\!\!\!\!\!\!\!\!\!\! \a_i(\ty) \!~=& \!\!\!\! \argmin \onenorm{\a} \!\! &\st ~ \twonorm{\ty - \X_i \a} \leq \frac{\gamma}{2} \varepsilon,\\
\label{eq:subL1noisy2} 
\!\!\!\!\!\!\!\!\!\!\!\! \a_{-i}(\ty) \!~=& \!\!\!\! \argmin \onenorm{\a} \!\! &\st ~ \twonorm{\tilde{\y} - \Y_{-i} \a} \leq \gamma \varepsilon.
\end{eqnarray}
In other words, in \eqref{eq:subL1noisy1}, we consider approximate reconstruction of $\ty$ using noise-free data points in $\S_i$, and in \eqref{eq:subL1noisy2}, we consider approximate reconstruction of $\ty$ using noisy data points in subspaces other than $\S_i^\varepsilon$.

The structure of our theoretical analysis in the paper is as follows. First, in Theorem \ref{thm:approximatessr2}, we find conditions based on the inradius and incoherence of subspaces under which we have $\onenorm{\a_i(\ty)} < \onenorm{\a_{-i}(\ty)}$, for every $\ty$. Our result corresponds to a novel generalization of the null-space property \cite{DonohoElad:PNAS03,Gribonval:TIT03, Stojnic:TSP09, VandenBerg:TIT10} to the case where 1) data lie in a union of subspaces, with the number of data points in each subspace typically larger than the subspace dimension; 2) all data points are corrupted by~noise. Then, in Theorems \ref{thm:approximatessr1} and \ref{thm:approx-support-recovery}, we show that if the noisy multi-subspace null-space property holds, \ie, $\onenorm{\a_i(\ty)} < \onenorm{\a_{-i}(\ty)}$, for every $\ty$, then the optimization problem \eqref{eq:L1-noisy} achieves approximate subspace-sparse recovery according to Definition \ref{def:approx-subspace-sparse-recovery}.

For brevity of the notation, we denote $\a_i(\ty)$ and $\a_{-i}(\ty)$ by $\a_i$ and $\a_{-i}$, respectively, whenever the argument $\ty$ is clear from the context. To characterize the set of admissible $\ty$ in our theoretical analysis, we make use of the following definition.
\vspace{2mm}
\begin{definition}
\label{def:admissiblety}
We denote by $\bW_i(\beta,\gamma, \varepsilon)$ the set of all $\ty$ with $\twonorm{\tilde{\y}} > \beta \varepsilon$, which can be written as the sum of a noise-free vector in $\S_i$ and a noise whose Euclidean norm is smaller than or equal to $0.5 \gamma \varepsilon$, \ie,
\begin{equation}
\label{def:bW}
\begin{split}
\!\! \bW_i(\beta, \gamma, \varepsilon) \triangleq \{ \ty \in \Re^n\!: ~&\twonorm{\ty} \geq \beta \varepsilon, ~\ty = \y + \z, \\ &~\y \in \S_i,~ \twonorm{\z} \leq 0.5 \gamma \varepsilon \}.
\end{split}
\end{equation}
\end{definition}
Next, we show that for a suitable value of $\gamma$, which depends on the subspace inradius, and for suitable values of $\beta$, the noisy multi-subspace null-space property holds.  
\vspace{2mm}
\begin{theorem}[Noisy Multi-Subspace Null-Space Property]
\label{thm:approximatessr2}
\emph{
Let $\gamma \triangleq \max_{i}{2 \, (1+\frac{2 \sqrt{2(\log N_i+\log n)}}{r_i})}$. Define $\beta$ as 
\begin{equation}
\label{eq:sufficientcondition}
\beta \triangleq ( 1 + \max_{i} \,  \frac{3 r_i}{r_i - (\mu_i + \varepsilon)} ) \, \frac{\gamma}{2} + \delta,
\end{equation}
where $\delta > 0$ is an arbitrarily small nonnegative number. Then, for every $\tilde{\y}$ which belongs to $\bW_i(\beta, \gamma, \varepsilon)$, the solutions of the optimization programs \eqref{eq:subL1noisy1} and \eqref{eq:subL1noisy2} satisfy 
\begin{eqnarray}
\onenorm{\a_i (\ty)} < \onenorm{\a_{-i} (\ty)}.
\end{eqnarray}
}
\end{theorem}
\vspace{2mm}
\begin{proof}
Consider $\ty$ in $\bW_i(\beta, \gamma, \varepsilon)$. We can write
\begin{equation}
\label{eq:ytildedef}
\tilde{\y} = \tilde{\x} + \tilde{\z},
\end{equation}
where from \eqref{def:bW}, we have $\tilde{\x} \in \S_i$ and $\twonorm{\tilde{\z}} \leq 0.5 \gamma \varepsilon$. Since $\twonorm{\ty} > \beta \varepsilon$, we have that $\twonorm{\tilde{\x}} > (\beta - 0.5 \gamma) \varepsilon$. We prove the result of the theorem in the following steps.
\smallskip\newline\noindent{\emph{Step 1:}} We find an upper bound on the $\ell_1$-norm of the solution of \eqref{eq:subL1noisy1} for $\tilde{\y}$, \ie, we show that
\begin{equation}
\onenorm{\a_i} \leq \frac{\twonorm{\tilde{\x}}}{r_i}.
\end{equation}
\smallskip\newline\noindent{\emph{Step 2:}} We find a lower bound on the $\ell_1$-norm of the solution of \eqref{eq:subL1noisy2} for $\tilde{\y}$, \ie, we show that, with high probability,
\begin{equation}
\frac{ \twonorm{\tilde{\x}} - 3 \gamma \varepsilon / 2 }{ \mu_i+\varepsilon } \leq \onenorm{\a_{-i}}.
\end{equation}
\smallskip\newline\noindent{\emph{Step 3:}} Combining the results of steps 1 and 2 and using the definition of $\beta$ in \eqref{eq:sufficientcondition}, we show that
\begin{equation}
\label{eq:suffcond2}
\onenorm{\a_i} \leq \frac{\twonorm{\tilde{\x}}}{r_i} < \frac{ \twonorm{\tilde{\x}} - 3 \gamma \varepsilon / 2 }{ \mu_i +\varepsilon } \leq \onenorm{\a_{-i}},
\end{equation}
obtaining the desired result. 
\medskip\newline\noindent{\emph{Proof of step 1:}} Our goal is to find an upper bound on the $\ell_1$-norm of the solution of \eqref{eq:subL1noisy1} for $\tilde{\y}$, defined in \eqref{eq:ytildedef}. Since $\tilde{\x}$ lies in $\S_i$, it can be written as a linear combination of noise-free data points in $\X_i$. Let
\begin{equation}
\b_i = \argmin \onenorm{\b} \quad \st \quad \tilde{\x} = \X_i \b.
\end{equation}
From Lemma \ref{lem:L1boundnonoise} we have $\onenorm{\b_i} \leq \frac{\twonorm{\tilde{\x}}}{r_i}$. In addition, using \eqref{eq:ytildedef}, we can write $\tilde{\y}$ as 
\begin{equation}
\ty = \tilde{\x} + \tilde{\z} = \X_i \b_i + \tilde{\z},
\end{equation} 
where $\twonorm{\tilde{\z}} \leq \gamma \varepsilon / 2$. As a result, $\b_i$ is a feasible solution for the $\ell_1$-minimization program in \eqref{eq:subL1noisy1}.  Hence, using the fact that $\a_i$ is the optimal solution of \eqref{eq:subL1noisy1}, we obtain
\begin{equation}
\onenorm{\a_i} \leq \onenorm{\b_i} \leq \frac{\twonorm{\tilde{\x}}}{r_i}.
\end{equation}
\newline\noindent{\emph{Proof of step 2:}} Our goal is to find a lower bound on the $\ell_1$-norm of the solution of \eqref{eq:subL1noisy2} for $\tilde{\y}$, defined in \eqref{eq:ytildedef}. By the feasibility of $\a_{-i}$ for the optimization program \eqref{eq:subL1noisy2}, we can write 
\begin{equation}
\ty = \Y_{-i} \a_{-i} + \v,
\end{equation}
where $\twonorm{\v} \leq \gamma \varepsilon$. Substituting the above equation into \eqref{eq:ytildedef}, we can write
\begin{equation}
\tilde{\x} = \Y_{-i} \a_{-i} + (\tilde{\z} - \v),
\end{equation}
where, $\twonorm{\tilde{\z}-\v} \leq 3 \gamma \varepsilon / 2$. Multiplying both sides of the above equation on the left by $\tilde{\x}^{\top} / \twonorm{\tilde{\x}}$ and using the H\"{o}lder's inequality, we obtain 
\begin{equation}
\begin{split}
\twonorm{\tilde{\x}} & \leq \infnorm{ \frac{\tilde{\x}^{\top}}{\twonorm{\tilde{\x}}} \Y_{-i} } \!\! \onenorm{\a_{-i}} + \frac{3}{2} \gamma \varepsilon \\
& \leq \bigg(\infnorm{ \frac{\tilde{\x}^{\top}}{\twonorm{\tilde{\x}}} \X_{-i} }\!\!\!\!+\infnorm{ \frac{\tilde{\x}^{\top}}{\twonorm{\tilde{\x}}} \Z_{-i} } \bigg)\onenorm{\a_{-i}} \\ 
& ~~~~~+ \frac{3}{2} \gamma \varepsilon ~~ \leq ~~ (\mu_i+\varepsilon) \onenorm{\a_{-i}} + \frac{3}{2} \gamma \varepsilon,
\end{split}
\end{equation}
where we used the fact that the Euclidean norm of each column of $\Z_{-i}\in\R^{n \times (N-N_i)}$ is at most $\varepsilon$, with high probability. Hence, we obtain the following lower bound on the optimal solution of \eqref{eq:subL1noisy2},
\begin{equation}
\frac{ \twonorm{\tilde{\x}} - 3 \gamma \varepsilon / 2 }{ \mu_i+\varepsilon } \leq \onenorm{\a_{-i}}.
\end{equation}
\newline\noindent{\emph{Proof of step 3:}} Using the definition of $\beta$ in \eqref{eq:sufficientcondition}, it is easy to verify that, we have
\begin{equation}
\frac{\mu_i+\varepsilon}{r_i} < 1 - \frac{3 \gamma}{2 \beta -  \gamma}.
\end{equation}
In addition, using the fact that $\twonorm{\tilde{\x}} \geq (\beta - 0.5 \gamma) \varepsilon$, we have
\begin{equation}
\frac{\mu_i+\varepsilon}{r_i} ~ < ~ 1 - \frac{3 \eta \varepsilon}{ (\beta -  \eta) \varepsilon } ~ \leq ~ \frac{ \twonorm{\tilde{\x}} - 3 \gamma \varepsilon / 2 }{ \twonorm{\tilde{\x}} },
\end{equation}
from which we obtain 
\begin{equation}
\label{eq:suffcond2}
\frac{\twonorm{\tilde{\x}}}{r_i} < \frac{ \twonorm{\tilde{\x}} - 3 \eta \varepsilon }{ \mu_i +\varepsilon }.
\end{equation}
Finally, combining \eqref{eq:suffcond2} with the results of steps 1 and 2, we obtain the desired result of the theorem, \ie,
\begin{equation}
\onenorm{\a_i} \leq \frac{\twonorm{\tilde{\x}}}{r_i} < \frac{ \twonorm{\tilde{\x}} - 3 \gamma \varepsilon / 2 }{ \mu_i +\varepsilon } \leq \onenorm{\a_{-i}}.
\end{equation}
\end{proof}
The result of Theorem \ref{thm:approximatessr2} shows that for a suitable value of the regularization parameter $\gamma$ and for a suitable $\beta$, the noisy multi-subspace null-space property $\onenorm{\a_i(\ty)} < \onenorm{\a_{-i}(\ty)}$ holds, for every $\ty \in \bW_i(\beta, \gamma, \varepsilon)$. In the next two theorems, we show that if the noisy multi-subspace null-space property holds, then the $\ell_1$-minimization program in \eqref{eq:L1-noisy}, with high probability, achieves approximate subspace-sparse recovery.
\vspace{1mm}
\begin{theorem}[Approximate Reconstruction]
\label{thm:approximatessr1}
\emph{
\noindent Let $\gamma \triangleq \max_{i}{2(1+\frac{2 \sqrt{2(\log N_i+\log n)}}{r_i})}$. Assume that there exists $\beta > 0.5 \gamma$ such that for every $\tilde{\y} \in \bW_i(\beta, \gamma, \varepsilon)$, the solutions of the optimization programs \eqref{eq:subL1noisy1} and \eqref{eq:subL1noisy2} satisfy $\onenorm{\a_i (\ty)} < \onenorm{\a_{-i} (\ty)}$. Then the solution of our proposed optimization program in \eqref{eq:L1-noisy}, with probability at least $1-\frac{1}{ (n N_i)^2 }$, satisfies
\begin{equation}
\label{eq:approxrecov1}
\twonorm{\y - \Y_i \c_i^*} \leq \beta \varepsilon.
\end{equation}
}
\end{theorem}
\begin{proof}
Let $\c^* = \begin{bmatrix} \c_i^* \\ \c_{-i}^* \end{bmatrix}$ be the solution of the $\ell_1$-minimization program \eqref{eq:L1-noisy}. For the sake of contradiction, assume that the condition in \eqref{eq:approxrecov1} does not hold, \ie, $\twonorm{\y - \Y_i \c_i^*} > \beta \varepsilon$. Since $\c^*$ is a feasible solution of the optimization program \eqref{eq:L1-noisy}, we can write
\begin{equation}
\label{eq:feasibility}
\y = \Y_i \c_i^* + \Y_{-i} \c_{-i}^* + \e,
\end{equation}
where $\twonorm{\e} \leq \gamma \varepsilon$. Define
\begin{equation}
\label{eq:ytilde1}
\tilde{\y} \triangleq \y - \Y_i \c_i^*.
\end{equation}
Note that by our assumption, we have $\twonorm{\ty} > \beta \varepsilon$. We arrive at contradiction by taking the following three steps.
\smallskip\newline\noindent {\emph{Step 1:}} Let $\a_{-i}$ be the solution of the optimization program \eqref{eq:subL1noisy2} for $\tilde{\y}$ defined in \eqref{eq:ytilde1}. We show that $\begin{bmatrix} \c_i^{*\top} & \a_{-i}^{\top} \end{bmatrix}^{\top}$ is a feasible solution of \eqref{eq:L1-noisy}, and satisfies 
\begin{equation}
\onenorm{ \begin{bmatrix} \c_i^{*} \\ \a_{-i} \end{bmatrix} } \leq \onenorm{ \begin{bmatrix} \c_i^{*} \\ \c_{-i}^* \end{bmatrix} }.
\end{equation}
\smallskip\newline\noindent{\emph{Step 2:}} Let $\a_{i}$ be the solution of the optimization program \eqref{eq:subL1noisy1} for $\tilde{\y}$ defined in \eqref{eq:ytilde1}. We show that $\begin{bmatrix} \c_i^{*\top}+\a_i^{\top} & \0 \end{bmatrix}^{\top}$ is a feasible solution of \eqref{eq:L1-noisy}.
\smallskip\newline\noindent{\emph{Step 3:}} Combining the results of the first two steps with the main assumption of the theorem, \ie, $\onenorm{\a_{i}} < \onenorm{\a_{-i}}$, we obtain 
\begin{equation}
\label{eq:L1bound1}
\onenorm{ \begin{bmatrix} \c_i^* + \a_i \\ \0 \end{bmatrix} } < \onenorm{ \begin{bmatrix} \c_i^* \\ \a_{-i} \end{bmatrix} } \leq \onenorm{ \begin{bmatrix} \c_i^* \\ \c_{-i}^* \end{bmatrix} }.
\end{equation}
contradicting the optimality of $\begin{bmatrix} \c_i^{*\top} & \c_{-i}^{*\top} \end{bmatrix}^{\top}$ for the optimization program \eqref{eq:L1-noisy}.
\medskip\newline \noindent{\emph{Proof of step 1:}} From \eqref{eq:feasibility}, we have $\tilde{\y} = \Y_{-i} \c_{-i}^* + \e$. In other words, $\tilde{\y}$ can be approximately written as a linear combination of noisy data points in $\Y_{-i}$. Since $\twonorm{\e} \leq \gamma \varepsilon$, we have that $\c_{-i}^*$ is a feasible solution of the optimization program \eqref{eq:subL1noisy2}.
Let $\a_{-i}$ be the optimal solution of \eqref{eq:subL1noisy2} for $\tilde{\y}$, hence 
\begin{equation}
\label{eq:intermed1}
\onenorm{\a_{-i}} \leq \onenorm{\c_{-i}^*}.
\end{equation}
We can write $\ty$ as
\begin{equation}
\label{eq:ytilde2}
\ty = \Y_{-i} \a_{-i} + \v,
\end{equation}
where $\twonorm{\v} \leq \gamma \varepsilon$. Using \eqref{eq:ytilde2} and the definition of $\tilde{\y}$ in \eqref{eq:ytilde1}, \ie, $\tilde{\y} = \y - \Y_i \c_i^*$, we can write $\y$ as
\begin{equation}
\label{eq:yrewritten}
\y = \Y_i \c_i^* + \Y_{-i} \a_{-i} + \v.
\end{equation}
Since $\twonorm{\v} \leq \gamma \varepsilon$, we have that $\begin{bmatrix} \c_i^* \\ \a_{-i} \end{bmatrix}$ is a feasible solution of the $\ell_1$-minimization program \eqref{eq:L1-noisy}. Thus, using \eqref{eq:intermed1}, we obtain the desired result of step 1, \ie, 
\begin{equation}
\label{eq:first-theory-step1}
\onenorm{\begin{bmatrix} \c_i^*\\ \a_{-i} \end{bmatrix}} \leq ~ \onenorm{\begin{bmatrix} \c_i^*\\ \c_{-i}^* \end{bmatrix}}.
\end{equation}
Another result that we use in the proof of step 2 is the fact that, with probability at least $1-\frac{1}{(n N_i)^2}$, we have
\begin{equation}
\label{eq:ctildebound}
\onenorm{\a_{-i}} \leq \onenorm{\begin{bmatrix} \c_i^*\\ \a_{-i} \end{bmatrix}} \leq \onenorm{\begin{bmatrix} \c_i^*\\ \c_{-i}^* \end{bmatrix}} \leq \frac{1}{r_i}.
\end{equation}
which follows from Lemma \ref{lem:L1bound}.
\medskip\newline\noindent{\emph{Proof of step 2:}} Since $\y \in \S_i^{\varepsilon}$, we can write $\y = \x + \z$, where $\x$ is a noise-free data point of unit Euclidean norm in $\S_i$ and $\z$ corresponds to noise whose Euclidean norm is bounded above by $\varepsilon$. Therefore, we can rewrite $\ty$ as
\begin{equation}
\ty = \y - \Y_i \c_i^* = (\x - \X_i \c_i^*) + (\z - \Z_i \c_i^*).
\end{equation}
Note that $\x - \X_i \c_i^*$ is a vector in $\S_i$, since it is a linear combination of noise-free data points in $\S_i$. Also, from Lemmas \ref{lem:noisebound} and \ref{lem:L1bound}, we have that $\twonorm{\z - \Z_i \c_i^*} \leq 0.5 \gamma \varepsilon$ holds with probability at least $1-\frac{1}{( n N_i )^2}$. Thus, $\tilde{\y}$ can be written as the sum of a vector in $\S_i$ plus a noise term whose Euclidean norm, with high probability, is bounded above by $0.5 \gamma \varepsilon$, hence, $\ty \in \bW_i(\beta,\gamma,\varepsilon)$. Thus, for $\tilde{\y}$, the optimization program \eqref{eq:subL1noisy1} has a feasible solution, which we denote by $\a_i$. Note that using the fact that $\ty \in \bW_i(\beta,\gamma,\varepsilon)$ and the assumption of the theorem, we have 
\begin{equation}
\label{eq:intermed2}
\onenorm{\a_i} < \onenorm{\a_{-i}}.
\end{equation}
By the optimality of $\tilde{\y}$ for the $\ell_1$-minimization program \eqref{eq:subL1noisy1}, we can write
\begin{equation}
\label{eq:ytilde4}
\tilde{\y} = \X_i \a_i + \v,
\end{equation}
where $\twonorm{\v} \leq 0.5 \gamma \varepsilon$. Using \eqref{eq:ytilde4} and the definition of $\tilde{\y}$ in \eqref{eq:ytilde1}, \ie, $\ty = \y - \Y_i \c_i^*$, we can write $\y$ as
\begin{equation}
\begin{split}
\y &= \Y_i \c_i^* + \X_i \a_i + \v \\ &= \Y_i (\c_i^* + \a_i) + (\v -\Z_i \a_i),
\end{split}
\end{equation}
where the second equality follows from the definition of $\X_i = \Y_i - \Z_i$. Thus, if $\twonorm{\v -\Z_i \a_i} < \gamma \varepsilon$, we have that $\begin{bmatrix} \c_i^* + \a_i \\ \0 \end{bmatrix}$ is a feasible solution of the optimization program \eqref{eq:L1-noisy}, hence obtaining the desired result of step 2. Notice that combining \eqref{eq:intermed2} and \eqref{eq:ctildebound}, with probability at least $1-\frac{1}{(n N_i)^2}$, we have $\onenorm{\a_{i}} < \frac{1}{r_i}$. Hence, using Lemma \ref{lem:noisebound}, the inequality $\twonorm{\v -\Z_i \a_i} < \gamma \varepsilon$ holds with high probability.
\medskip\newline\noindent{\emph{Proof of step 3:}} Based on the assumption of the theorem, since $\ty \in \bW_i(\beta,\gamma,\varepsilon)$, we have that $\onenorm{\a_i} < \onenorm{\a_{-i}}$. Hence, using the results of steps 1 and 2, we obtain 
\begin{equation}
\label{eq:L1bound1}
\begin{split}
\onenorm{ \begin{bmatrix} \c_i^* + \a_i \\ \0 \end{bmatrix} } &\leq \onenorm{\c_i^*} + \onenorm{\a_i} \\ &< \onenorm{ \begin{bmatrix} \c_i^* \\ \a_{-i} \end{bmatrix} } \leq \onenorm{ \begin{bmatrix} \c_i^* \\ \c_{-i}^* \end{bmatrix} }.
\end{split}
\end{equation}
This contradicts the optimality of $\begin{bmatrix} \c_i^* \\ \c_{-i}^* \end{bmatrix}$ for the optimization program \eqref{eq:L1-noisy}. Hence, we must have
\begin{equation}
\twonorm{\y - \Y_i \c_i^*} \leq \beta \varepsilon.
\end{equation}
\end{proof}
Up to this point, we have shown that, under appropriate conditions, for any noisy data point $\y \in \S_i^\varepsilon$, the solution of the $\ell_1$-minimization program \eqref{eq:L1-noisy} is such that $\y$ will be reconstructed with high accuracy using noisy data points from its own subspace. Next, we show that, in the optimal solution, the coefficients corresponding to data points in other subspaces will be sufficiently small, provided that the noise level $\varepsilon$ is not very large. More specifically, we prove the following result.
\vspace{2mm}
\begin{theorem}[Approximate Support Recovery]
\label{thm:approx-support-recovery}
\emph{
Let $\c^{*\top} = \begin{bmatrix} \c_i^{*\top} & \c_{-i}^{*\top} \end{bmatrix}$ be the solution of the optimization program in \eqref{eq:L1-noisy} for a noisy data point $\y$ in $\S_i^{\varepsilon}$. Assume that $\varepsilon \leq \frac{\gamma}{2 \beta + \gamma} r_i$ and that the approximate reconstruction condition $\twonorm{\y - \Y_i \c_i^*} \leq \beta \varepsilon$ holds. Then, we have
\begin{equation}
\label{eq:approx-supp-recov1}
\onenorm{\c_{-i}^*} \leq \frac{\beta+\gamma / 2}{r_i} \varepsilon.
\end{equation}
}
\end{theorem}
\begin{proof}
Since $\y$ is a noisy data point in $\S_i^{\varepsilon}$, we can write $\y = \x + \z$, where $\x$ is a noise-free data point in $\S_i$ and $\z$ corresponds to noise whose Euclidean norm is smaller than or equal to $\varepsilon$. Define $\tilde{\y} \triangleq \y - \Y_i \c_i^*$, hence, from the assumption of the theorem, we have $\twonorm{\tilde{\y}} \leq \beta \varepsilon$. We can write $\tilde{\y}$ as
\begin{equation}
\label{eq:ytildedecomp}
\tilde{\y}=\y -\Y_i \c_i^*=\underbrace{( \x - \X_i \c_i^*)}_{\triangleq \, \tilde{\x}}+\underbrace{( \z - \Z_i \c_i^*)}_{\triangleq \, \tilde{\z}}.
\end{equation}
Notice that $\tilde{\x}$ is a vector in $\S_i$, since it is a linear combination of noise-free data points in $\S_i$, and $\tilde{\z}$ corresponds to noise whose Euclidean norm is bounded as $\twonorm{\tilde{\z}} \leq 0.5 \gamma \varepsilon$, using Lemmas \ref{lem:noisebound} and \ref{lem:L1bound}. We prove the result in \eqref{eq:approx-supp-recov1} by taking the following three steps.
\smallskip\newline\noindent{\emph{Step 1:}} First, we show that the minimum $\ell_1$-norm of representing $\tilde{\x}$, the noise-free part of $\tilde{\y}$, using $\X_i$, the noise-free data points in $\S_i$, is bounded by
\begin{eqnarray}
\label{eq:3step1}
&\min~ \onenorm{\b} \quad \quad  \leq \quad \frac{2 \beta+\gamma}{2 r_i} \, \varepsilon \nonumber\\
&\hspace{-16mm}\st \quad \tilde{\x} = \X_i \b.
\end{eqnarray}
\newline\noindent{\emph{Step 2:}} Next, we show that the minimum $\ell_1$-norm of the approximate representation of $\tilde{\y}$ in terms of noisy data points in $\S_i^\varepsilon$, \ie, $\Y_i$, is bounded by
\begin{eqnarray}
\label{eq:3step2}
\!\!\!\!\!\!\!\!\!\!\!\!\!\!\!\!\!\! &\min \onenorm{\b} \qquad \qquad \qquad \leq \quad &\min \onenorm{\b}\nonumber\\
\!\!\!\!\!\!\!\!\!\!\!\!\!\!\!\!\!\! &\hspace{-6mm}\st \quad \twonorm{\tilde{\y} - \Y_i \b} \leq \gamma \varepsilon &\st \quad \tilde{\x} = \X_i \b.
\end{eqnarray}
\newline\noindent{\emph{Step 3:}} Finally, we prove that, for the solution of the optimization program \eqref{eq:L1-noisy}, the $\ell_1$-norm of the coefficients corresponding to noisy data points in subspaces other than $\S_i^\varepsilon$, \ie, $\onenorm{\c_{-i}^*}$, is bounded by
\begin{eqnarray}
\label{eq:3step3}
\!\!\!\!\!\!\!\!\!\!\!\!\!\!\!\!\!\!\!\!\!\!\!\!\!\!\!\!\!\!\!\!\!\!\!\!\!\!\! \qquad \qquad \qquad \onenorm{\c_{-i}^*} & \quad \leq \quad &\min \onenorm{\b}\nonumber\\
\!\!\!\!\!\! &\text{ }& \st ~~ \twonorm{\tilde{\y} - \Y_i \b} \leq \gamma \varepsilon.
\end{eqnarray}
Combining the results of steps $1$ to $3$, we obtain \eqref{eq:approx-supp-recov1}.
\medskip\newline\noindent{\emph{Proof of step 1:}}
Let  $\b_{i}$ be the solution of the $\ell_1$-minimization program
\begin{equation}
\label{eq:bi-step1}
\b_i \; = \; \argmin \onenorm{\b} \quad \st \quad \tilde{\x} = \X_i \b. 
\end{equation}
Using \eqref{eq:ytildedecomp}, we can write $\tilde{\x} = \tilde{\y} - \tilde{\z}$, where $\twonorm{\tilde{\y}} \leq \beta \varepsilon$ and $\twonorm{\tilde{\z}} \leq 0.5 \gamma \varepsilon$. As a result, the Euclidean norm of $\tilde{\x}$ is bounded by $\twonorm{\tilde{\x}} \leq (\beta+ \gamma / 2)\varepsilon$. 
Thus, using Lemma \ref{lem:L1boundnonoise}, we obtain
\begin{equation}
\label{eq:step1result}
\onenorm{\b_i} \leq \frac{\twonorm{\tilde{\x}}}{r_i} \leq \frac{\beta+\eta}{r_i} \, \varepsilon.
\end{equation}
\newline\noindent{\emph{Proof of step 2:}}
Let $\b_{i}$ be the solution of the optimization program \eqref{eq:bi-step1}, hence, 
\begin{equation}
\label{eq:interm11}
\tilde{\x} = \X_i \b_i.
\end{equation}
Since, using \eqref{eq:ytildedecomp}, we have $\tilde{\x}=\tilde{\y}-\tilde{\z}$, and also $\X_i = \Y_i - \Z_i$, we can rewrite \eqref{eq:interm11} as
\begin{equation}
\tilde{\y} = \Y_i \b_{i} + ( \tilde{\z} - \Z_i \b_{i} ).
\end{equation}
Thus, if we show that $\twonorm{ \tilde{\z} - \Z_i \b_{i} } \leq \gamma \varepsilon$, then we obtain \eqref{eq:3step2}, since $\b_{i}$ is the optimal solution of the right hand-side of \eqref{eq:3step2}, while it is also a feasible solution of the left hand-side of \eqref{eq:3step2}. Notice that the columns of $\X_i$ are data points in a $d_i$-dimensional subspace of $\Re^n$. As a result, from the linear programming theory, the optimal solution of the right hand-side of \eqref{eq:3step2}, $\b_{i}$, has a support whose size is at most $d_i$. Thus, using the fact that the Euclidean norm of the columns of $\Z_i$ is at most $\varepsilon$, we have $\twonorm{\Z_i \b_{i}} \leq \varepsilon \onenorm{\b_{i}} $. Now, using the result of step 1 in \eqref{eq:step1result}, \ie, $\onenorm{\b_{i}} \leq \frac{\beta+\eta}{r_i} \, \varepsilon$, and the assumption of the theorem on the noise level, \ie, $\varepsilon \leq \frac{\gamma}{2 \beta+\gamma} \, r_i$, we obtain
\begin{equation}
\begin{split}
\twonorm{\tilde{\z} - \Z_i \b_{i}} &\leq \twonorm{\tilde{\z}}+\twonorm{\Z_i \b_{i}} \leq \eta \varepsilon+\varepsilon\onenorm{\b_{i}}\\ &\leq \eta \varepsilon + \frac{\beta+\eta}{r_i}\varepsilon^2 \leq 2 \eta \varepsilon.
\end{split}
\end{equation}
\newline\noindent{\emph{Proof of step 3:}}
Let $\b'_i$ be the optimal solution of the right hand-side of \eqref{eq:3step3}, \ie, 
\begin{equation}
\b'_{i} \; = \; \argmin \onenorm{\b} \quad \st \quad \twonorm{\tilde{\y} - \Y_i \b} \leq \gamma \varepsilon. 
\end{equation} 
For the sake of contradiction, assume that the inequality in \eqref{eq:3step3} does not hold, so we have $\onenorm{\b'_{i}} < \onenorm{\c_{-i}^*}$. Using the definition of $\tilde{\y}$ in \eqref{eq:ytildedecomp}, \ie, $\tilde{\y} = \y - \Y_i \c_i^*$, we have
\begin{equation}
\twonorm{\tilde{\y} - \Y_i \b'_{i}} = \twonorm{\y - \Y_i (\c_i^{*} + \b'_{i})} \leq \gamma \varepsilon.
\end{equation}
As a result, $\begin{bmatrix} \c_i^*+ \b'_{i} \\ \0 \end{bmatrix}$ is a feasible solution of the optimization problem \eqref{eq:L1-noisy}. Moreover, we have
\begin{equation}
\onenorm{\begin{bmatrix} \c_i^*+ \b'_{i} \\ \0 \end{bmatrix}} \leq \onenorm{\c_i^*} + \onenorm{\b'_{i}} < \onenorm{\begin{bmatrix} \c_i^* \\ \c_{-i}^*\end{bmatrix}},
\end{equation}
which contradicts the optimality of $\begin{bmatrix} \c_i^* \\ \c_{-i}^* \end{bmatrix}$ for \eqref{eq:L1-noisy}. Thus, we must have $\onenorm{\c_{-i}^*} \leq \onenorm{\b'_{i}}$.
\end{proof}
\vspace{2mm}
Putting the results of Theorems \ref{thm:approximatessr2}, \ref{thm:approximatessr1} and \ref{thm:approx-support-recovery} together, we arrive at our main theoretical results, guaranteeing approximate subspace-sparse recovery in the presenese of noise using the $\ell_1$-minimization program in \eqref{eq:L1-noisy}.
\vspace{1mm}
\begin{theorem}
\label{thm:maintheory-repeat}
\emph{
Assume that the columns of $\Y \in \Re^{n \times N}$ correspond to noisy data points lying in $\{ \S_i^\varepsilon \}_{i=1}^{L}$, with $N_i$ data points in each $\S_i^\varepsilon$. Consider the $\ell_1$-minimization program in \eqref{eq:L1-noisy} with the $\gamma$ defined as 
\begin{equation}
\label{eq:gamma-choice}
\gamma \triangleq \max_{i}{2(1+\frac{2 \sqrt{2(\log N_i+\log n)}}{r_i})}.
\end{equation}
Define $\beta$ as 
\begin{equation}
\label{eq:beta-choice}
\beta \triangleq ( 1 + \max_{i} \,  \frac{3 r_i}{r_i - (\mu_i + \varepsilon)} ) \, \frac{\gamma}{2} + \delta,
\end{equation}
where $\delta > 0$ is arbitrarily small.
Then, for every $i \in \{1, \ldots, L\}$ and every $\y \in \S_i^{\varepsilon}$, the solution of the optimization problem in \eqref{eq:L1-noisy}, with probability at least $1-\frac{1}{(n N_i)^2}$, satisfies 
\begin{equation}
\label{eq:approximate-reconstruction-result}
\twonorm{\y - \Y_i \c_i^*} \leq \beta \varepsilon.
\end{equation}
In addition, assume that $\varepsilon \leq \frac{\gamma}{2 \beta + \gamma} r_i$. Then, we have that
\begin{equation}
\label{eq:support-detection-result}
\onenorm{\c_{-i}^*} \leq \frac{2 \beta+\gamma}{2 r_i} \varepsilon
\end{equation}
holds with probability at least $1-\frac{1}{(n N_i)^2}.$
}
\end{theorem}
\vspace{2mm}
\begin{proof}
Given the choice of $\gamma$ in \eqref{eq:gamma-choice} and $\beta$ in \eqref{eq:beta-choice}, from Theorem \ref{thm:approximatessr2},we have that the multi-subspace noisy null-space property holds, \ie, for every $\ty$ in $\bW_i(\beta,\gamma,\varepsilon)$, we have $\onenorm{\a_i(\ty)} < \onenorm{\a_{-i}(\ty)}$, where $\a_{i}(\ty)$ and $\a_{-i}(\ty)$ denote the solutions of \eqref{eq:subL1noisy1} and \eqref{eq:subL1noisy2}, respectively. As a result, the condition of the Theorem \ref{thm:approximatessr1} is satisfied and we have that \eqref{eq:approximate-reconstruction-result} holds, with high probability. Finally, given the approximate reconstruction condition and the assumption of the theorem on the maximum value of $\varepsilon$, from Theorem \ref{thm:approx-support-recovery}, we have that \eqref{eq:support-detection-result} holds, with high probability.
\end{proof}
\vspace{2mm}
Notice that in all of our theoretical results so far we allow for arbitrary subspace arrangements and data distributions in subspaces, without any randomness assumption. In fact, assuming random distribution for data points, we can further show that in the solution of the $\ell_1$-minimization program  \eqref{eq:L1-noisy}, the coefficients from the correct support, \ie, $\c_i^*$, not only reconstruct $\y$ with a high accuracy, but also have sufficiently large values. More specifically, we prove the following result.


\vspace{2mm}
\begin{theorem}[Correct Support Detection]
Assume that the noise-free data in each subspace $\S_i$, \ie, the columns of $\X_i$, are drawn uniformly at random from the intersection of the unit hypersphere with $\S_i$. Let $\c^{*\top} = \begin{bmatrix} \c_i^{*\top} & \c_{-i}^{*\top} \end{bmatrix}$ be the solution of the optimization program in \eqref{eq:L1-noisy} for a noisy data point $\y$ in $\S_i^{\varepsilon}$. Assume that the approximate reconstruction condition $\twonorm{\y - \Y_i \c_i^*} \leq \beta \varepsilon$ holds.
Then, with probability at least $1-\frac{2}{N_i^2}$, we have 
\begin{equation}
\label{eq:approxSuppRecov2}
\onenorm{\c_i^*} \geq \frac{1 - (\beta + 1) \, \varepsilon}{2 \sqrt{ \frac{2\log N_i}{d_i} } + 2 \sqrt{ \frac{2\log N_i}{n} } \, \varepsilon}.
\end{equation}
\end{theorem}
\begin{proof}
Our goal is to find a lower bound on the $\ell_1$-norm of $\c_i^*$. Since $\y$ belongs to $\S_i^{\varepsilon}$, it can be written as $\y = \x + \z$, where $\x$ is a vector of unit Euclidean norm in $\S_i$ and $\z$ corresponds to noise, where $\twonorm{\z} \leq \varepsilon$. Using the assumption of the theorem, \ie, $\twonorm{\y - \Y_i \c_i^*} \leq \beta \varepsilon$, we can write
\begin{equation}
\y = \Y_i \c_i^* + \e,
\end{equation}
where $\twonorm{\e} \leq \beta \varepsilon$. Since $\Y_i = \X_i + \Z_i$, we can rewrite the above equation as
\begin{equation}
\label{eq:xzzt}
\x + \z - \e = (\X_i + \Z_i) \c_i^*.
\end{equation}
Multiplying both sides of \eqref{eq:xzzt} from left by $\x^{\top}$, and taking the absolute values, we have
\begin{equation}
\label{eq:xzzt2}
\abs{ \x^{\top} (\x + \z - \e) } = \abs{ \x^{\top} (\X_i + \Z_i) \c_i^* }.
\end{equation}
Note that the left hand-side of \eqref{eq:xzzt2} is bounded by
\begin{equation}
\label{eq:bounds1}
1- (\beta + 1) \, \varepsilon ~ \leq ~ \abs{ \x^{\top} (\x + \z - \e) }.
\end{equation}
On the other hand, using the H\"older's inequality, the right hand-side of \eqref{eq:xzzt2}, with probability at least $1 - \frac{2}{N_i^2}$, is bounded by
\begin{equation}
\begin{split}
\label{eq:bounds2}
\abs{ \x^{\top} (\X_i + \Z_i) \c_i^* } & \leq \infnorm{\x^{\top}(\X_i+\Z_i)} \onenorm{\c_i^*}\\
& \leq (\infnorm{\x^{\top} \X_i} \!\! + \infnorm{\x^{\top} \Z_i})\onenorm{\c_i^*}\\
& \leq 2 ( \sqrt{ \frac{2\log N_i}{d_i} } + \sqrt{ \frac{2\log N_i}{n} } \, \varepsilon) \onenorm{\c_i^*}.
\end{split}
\end{equation}
The last inequality in the above follows from Lemma \ref{lem:innerproduct-bound} in the Appendix. Finally, using the lower-bound in \eqref{eq:xzzt2} and the upper-bound in \eqref{eq:bounds1}, for \eqref{eq:bounds2}, we obtain
\begin{equation}
1 - (\beta + 1) \, \varepsilon \leq 2 ( \sqrt{\frac{2\log N_i}{d_i}} + \sqrt{ \frac{2\log N_i}{n} } \, \varepsilon) \onenorm{\c_i^*},
\end{equation}
hence, we arrive at our desired result in \eqref{eq:approxSuppRecov2}.
\end{proof}
%

\section{Conclusions}
\label{sec:conclusion}
In this paper, we considered the problem of finding sparse representations for noisy data points in a dictionary that consists of corrupted data lying close to a union of subspaces. More specifically, we assumed that the columns of the dictionary correspond to data points drawn from a union of subspaces and corrupted by Gaussian noise whose Euclidean norm is about $\varepsilon$. We studied a constrained $\ell_1$-minimization program and showed that under appropriate conditions on the subspace-inradius and subspace-coherence parameters, the solution of the proposed optimization recovers a solution satisfying approximate subspace-sparse recovery. In other words, we showed that a noisy data point will be reconstructed using data points from its underlying subspace with an error that is of the order of $O(\varepsilon)$, while coefficients corresponding to data points in other subspaces are sufficiently small, of the order of $O(\varepsilon)$. To achieve this result, we developed an analysis framework based on a novel generalization of the null-space property to the setting where data lie in multiple subspaces, the number of data points in each subspace exceeds the dimension of the subspace and all data points are corrupted by noise. Finally, assuming random distribution for data points, we further showed that, in the solution of the constrained optimization, coefficients from the desired support not only reconstruct a given point with high accuracy, but also have sufficiently large values, \ie, are of the order of $O(1)$. 


\appendices

\section{}
In the paper, we used the fact that for a Gaussian random vector $\z \in\R^n$ with i.i.d entries drawn from $\N(0,\frac{\epsilon^2}{n})$, with high probability, we have $\twonorm{\z} \leq \varepsilon$. To see this, notice that if $z_i \sim \N(0,\frac{\epsilon^2}{n})$, then $q_i \triangleq \frac{n}{\varepsilon^2} z_i^2$ follows a $\chi^2$ distribution with one degree of freedom. We use the following Lemma from \cite{Laurent:ASTAT00}, which provides a bound on linear combination of $\chi^2$ random variables.
\begin{lemma}
Let $q_1, \ldots, q_n$ be independent $\chi^2$ random variables, each with one degree of freedom. For any vector $\a = \begin{bmatrix} a_1 \!&\! \cdots \!&\! a_n \end{bmatrix}^\top \in \Re^n_+$ with nonnegative entries, and for any $t > 0$, we have 
\begin{equation}
\text{Pr} \left [ \sum_{i=1}^{n}{ a_i q_i } > \onenorm{\a} + 2 \sqrt{t} \twonorm{\a} + 2 t \, \infnorm{\a} \right ] \leq e^{-t}.
\end{equation}
\end{lemma}
If we set $a_i = \varepsilon^2/n$ for all $i = 1, \ldots, n$, then using the above lemma, we obtain that
\begin{equation}
\text{Pr} \left [ \twonorm{\z}^2 > \varepsilon^2 (1 + \rho)^2 \right ] \leq e^{- \frac{(1 + \rho)^2 - \sqrt{2 (1 + \rho)^2 - 1}}{2} n }
\end{equation}
holds for any $\rho > 0$.

We also have the following Lemma, which provides a bound on the inner product between a fixed vector and a matrix of Gaussian random variables.
\begin{lemma}
\label{lem:innerproduct-bound}
Assume $\A \in \Re^{m \times N}$ has i.i.d entries drawn from $\N(0,\sigma^2)$. Let $\x \in \Re^{m}$ be a vector of unit Euclidean norm. We have
\begin{equation}
\text{Pr} \left [ \infnorm{ \A^\top \z } \leq 2 \sqrt{ 2 \log{N} } \, \sigma \right ] \geq 1 - \frac{1}{N^2}.
\end{equation}
\end{lemma}

\section{Proof of Lemma \ref{lem:noisebound}}
Denote the $j$-th row of the noise matrix $\Z_i$ by $\Z_i^{(j)} \in \R^{N_i}$. We can write
\begin{equation}
\label{eq:noisebound1}
\twonorm{\Z_i \c_i}^2 \leq \sum_{j=1}^{n}\langle \Z_i^{(j)},\c_i \rangle^2 \leq \sum_{j=1}^{n}\infnorm{\Z_i^{(j)}}^2\onenorm{\c_i}^2 
\end{equation}
Since each entry of $\Z_i$ has a standard deviation of $\frac{\varepsilon}{\sqrt{n}}$, with probability at least $1-\frac{1}{(n N_i)^2}$, we have
\begin{equation}
\label{eq:noisebound2}
\infnorm{\Z_i^{(j)}} \leq 2 \varepsilon \, \sqrt{\frac{2 (\log N_i+\log n)}{n}}, \quad \forall j \in \{1, \ldots, n\}.
\end{equation}
Substituting \eqref{eq:noisebound2} into \eqref{eq:noisebound1}, we have that, with probability at least $1-\frac{1}{(n N_i)^2}$, the following inequality holds,
\begin{equation}
\twonorm{\Z_i \c_i}\le 2 \varepsilon \sqrt{ 2 (\log N_i+\log n) } \onenorm{\c_i}.
\end{equation}

\section{Proof of Lemma \ref{lem:L1bound}}
Note that $\y$ can be written as $\y = \x + \z$, where $\x \in \S_i$ has unit Euclidean norm and $\twonorm{\z} \leq \varepsilon$. Notice that $\x$ can be written as a linear combination of noise-free data points in $\S_i$. Let
\begin{equation}
\c_i^*=\argmin \onenorm{\c} \quad \st \quad \x = \X_i \c.
\end{equation}
Then from Lemma \ref{lem:L1boundnonoise}, we have $\onenorm{\c_i^*} \leq \frac{1}{r_i}$. On the other hand, we can rewrite $\x = \X_i \c_i^*$ as
\begin{equation}
\y - \z = (\Y_i - \Z_i) \c_i^*, 
\end{equation}
from which we obtain,
\begin{equation}
\y = \Y_i \c_i^* + (\z - \Z_i \c_i^*).
\end{equation}
From Lemma \ref{lem:noisebound}, with probability at least $1-\frac{1}{(n N_i)^2}$, we have
\begin{equation}
\twonorm{ \z - \Z_i \c_i^*} ~ \leq ~ \varepsilon \, (1 + \frac{2\sqrt{2 (\log N_i+\log n )}}{r_i}).
\end{equation}
As a result, with high probability, $\begin{bmatrix} \c_i^* \\ \0 \end{bmatrix}$ is a feasible solution of the optimization program \eqref{eq:L1noisy}. Thus, we must have 
\begin{equation}
\onenorm{\c^*} \leq \onenorm{\c_i^*} \leq \frac{1}{r_i}.
\end{equation}

\ifCLASSOPTIONcaptionsoff
  \newpage
\fi

{
\bibliographystyle{IEEEtran}
\bibliography{biblio/math,biblio/sparse,biblio/vidal,biblio/learning,biblio/vision,biblio/recognition,biblio/segmentation,biblio/control}

\begin{thebibliography}{10}
\providecommand{\url}[1]{#1}
\csname url@samestyle\endcsname
\providecommand{\newblock}{\relax}
\providecommand{\bibinfo}[2]{#2}
\providecommand{\BIBentrySTDinterwordspacing}{\spaceskip=0pt\relax}
\providecommand{\BIBentryALTinterwordstretchfactor}{4}
\providecommand{\BIBentryALTinterwordspacing}{\spaceskip=\fontdimen2\font plus
\BIBentryALTinterwordstretchfactor\fontdimen3\font minus
  \fontdimen4\font\relax}
\providecommand{\BIBforeignlanguage}[2]{{%
\expandafter\ifx\csname l@#1\endcsname\relax
\typeout{** WARNING: IEEEtran.bst: No hyphenation pattern has been}%
\typeout{** loaded for the language `#1'. Using the pattern for}%
\typeout{** the default language instead.}%
\else
\language=\csname l@#1\endcsname
\fi
#2}}
\providecommand{\BIBdecl}{\relax}
\BIBdecl

\bibitem{Basri:PAMI03}
R.~Basri and D.~Jacobs, ``Lambertian reflection and linear subspaces,''
  \emph{{IEEE} Transactions on Pattern Analysis and Machine Intelligence},
  vol.~25, no.~3, pp. 218--233, 2003.

\bibitem{Tomasi:IJCV92}
C.~Tomasi and T.~Kanade, ``Shape and motion from image streams under
  orthography,'' \emph{International Journal of Computer Vision}, vol.~9,
  no.~2, pp. 137--154, 1992.

\bibitem{Hastie:StatSci98}
T.~Hastie and P.~Simard, ``Metrics and models for handwritten character
  recognition,'' \emph{Statistical Science}, vol.~13, no.~1, pp. 54--65, 1998.

\bibitem{Hong:TIP06}
W.~Hong, J.~Wright, K.~Huang, and Y.~Ma, ``Multi-scale hybrid linear models for
  lossy image representation,'' \emph{IEEE Trans. on Image Processing},
  vol.~15, no.~12, pp. 3655--3671, 2006.

\bibitem{Yang:CVIU08}
A.~Yang, J.~Wright, Y.~Ma, and S.~Sastry, ``Unsupervised segmentation of
  natural images via lossy data compression,'' \emph{Computer Vision and Image
  Understanding}, vol. 110, no.~2, pp. 212--225, 2008.

\bibitem{Elhamifar:TPAMI13}
E.~Elhamifar and R.~Vidal, ``Sparse subspace clustering: Algorithm, theory, and
  applications,'' \emph{{IEEE} Transactions on Pattern Analysis and Machine
  Intelligence}, 2013.

\bibitem{Wright:PAMI09}
J.~Wright, A.~Yang, A.~Ganesh, S.~Sastry, and Y.~Ma, ``Robust face recognition
  via sparse representation,'' \emph{{IEEE} Transactions on Pattern Analysis
  and Machine Intelligence}, vol.~31, no.~2, pp. 210--227, Feb. 2009.

\bibitem{Elhamifar:TSP12}
E.~Elhamifar and R.~Vidal, ``Block-sparse recovery via convex optimization,''
  \emph{{IEEE} Transactions on Signal Processing}, 2012.

\bibitem{Elhamifar:CVPR09}
------, ``Sparse subspace clustering,'' in \emph{{IEEE} Conference on Computer
  Vision and Pattern Recognition}, 2009.

\bibitem{Lerman:Annals11}
G.~Lerman and T.~Zhang, ``Robust recovery of multiple subspaces by geometric lp
  minimization,'' \emph{Ann. Statist.}, vol.~39, no.~5, 2011.

\bibitem{Lerman:CA14}
------, ``lp -recovery of the most significant subspace among multiple
  subspaces with outliers,'' \emph{Constructive Approximation}, vol.~40, no.~3,
  2014.

\bibitem{Favaro:CVPR11}
P.~Favaro, R.~Vidal, and A.~Ravichandran, ``A closed form solution to robust
  subspace estimation and clustering,'' in \emph{{IEEE} Conference on Computer
  Vision and Pattern Recognition}, 2011.

\bibitem{DelaTorre:CVPR12}
R.~Liu, Z.~Lin, and Z.~S. F.~DelaTorre, ``Fixed-rank representation for
  unsupervised visual learning,'' \emph{CVPR}, 2012.

\bibitem{Heckel:ISIT13}
R.~Heckel and H.~Bölcskei, ``Noisy subspace clustering via thresholding,''
  \emph{IEEE International Symposium on Information Theory (ISIT)}, pp.
  1382--1386, 2013.

\bibitem{Elhamifar:CVPR12}
E.~Elhamifar, G.~Sapiro, and R.~Vidal, ``See all by looking at a few: Sparse
  modeling for finding representative objects,'' in \emph{{IEEE} Conference on
  Computer Vision and Pattern Recognition}, 2012.

\bibitem{Esser:TIP12}
E.~Esser, M.~Moller, S.~Osher, G.~Sapiro, and J.~Xin, ``A convex model for
  non-negative matrix factorization and dimensionality reduction on physical
  space,'' \emph{IEEE Transactions on Image Processing}, vol.~21, no.~7, pp.
  3239--3252, 2012.

\bibitem{Elhamifar:NIPS12}
E.~Elhamifar, G.~Sapiro, and R.~Vidal, ``Finding exemplars from pairwise
  dissimilarities via simultaneous sparse recovery,'' \emph{Neural Information
  Processing Systems}, 2012.

\bibitem{Eldar:TIT09}
Y.~C. Eldar and M.~Mishali, ``Robust recovery of signals from a structured
  union of subspaces,'' \emph{IEEE Trans. Inform. Theory}, vol.~55, no.~11, pp.
  5302--5316, 2009.

\bibitem{Montanari:UAI12}
A.~Zhang, N.~Fawaz, S.~Ioannidis, and A.~Montanari, ``Guess who rated this
  movie: Identifying users through subspace clustering,'' \emph{Uncertainty in
  Artificial Intelligence (UAI)}, 2012.

\bibitem{Lerman:ICCV13}
X.~Wang, S.~Atev, J.~Wright, and G.~Lerman, ``Fast subspace search via
  grassmannian based hashing,'' \emph{International Conference of Computer
  Vision (ICCV)}, 2013.

\bibitem{Sapiro:ICLR14}
Q.~Qiu and G.~Sapiro, ``Learning transformations for classification forests,''
  \emph{International Conference on Learning Representations (ICLR)}, 2014.

\bibitem{Donoho:CPAM06}
D.~L. Donoho, ``For most large underdetermined systems of linear equations the
  minimal $\ell^1$-norm solution is also the sparsest solution,''
  \emph{Communications on Pure and Applied Mathematics}, vol.~59, no.~6, pp.
  797--829, 2006.

\bibitem{Candes-Tao:TIT05}
E.~J. Cand\`es and T.~Tao, ``Decoding by linear programming,'' \emph{IEEE
  Trans. on Information Theory}, vol.~51, no.~12, pp. 4203--4215, 2005.

\bibitem{Tibshirani:RSS96}
R.~Tibshirani, ``Regression shrinkage and selection via the {LASSO},''
  \emph{Journal of the Royal Statistical Society B}, vol.~58, no.~1, pp.
  267--288, 1996.

\bibitem{Elad:SIAM09}
A.~Bruckstein, D.~Donoho, and M.~Elad, ``From sparse solutions of systems of
  equations to sparse modeling of signals and images,'' \emph{SIAM Review},
  vol.~51, no.~1, pp. 34--81, Feb. 2009.

\bibitem{Yu:AnnalStat09}
N.~Meinshausen and B.~Yu, ``Lasso-type recovery of sparse representations for
  high-dimensional data,'' \emph{Annals of Statistics}, vol.~37, no.~1, pp.
  246--270, 2009.

\bibitem{Candes:TIT06}
E.~J. Cand\`es, J.~Romberg, and T.~Tao, ``Robust uncertainty principles: Exact
  signal reconstruction from highly incomplete frequency information,''
  \emph{IEEE Transactions on Information Theory}, vol.~52, no.~2, pp. 489--509,
  2006.

\bibitem{Candes:SPM08}
E.~J. Cand\`es and M.~Wakin, ``An introduction to compressive sampling,''
  \emph{IEEE Signal Processing Magazine}, vol.~25, no.~2, pp. 21--30, Mar.
  2008.

\bibitem{Elhamifar:Annals14}
M.~Soltanolkotabi, E.~Elhamifar, and E.~J. Candes, ``Robust subspace
  clustering,'' \emph{Annals of Statistics}, 2014.

\bibitem{Elhamifar:ICASSP10}
E.~Elhamifar and R.~Vidal, ``Clustering disjoint subspaces via sparse
  representation,'' in \emph{{IEEE} International Conference on Acoustics,
  Speech, and Signal Processing}, 2010.

\bibitem{Candes:ASTAT12}
M.~Soltanolkotabi and E.~J. Candes, ``A geometric analysis of subspace
  clustering with outliers,'' \emph{Annals of Statistics}, 2012.

\bibitem{Leng:NIPS13}
Y.~Wang, H.~Xu, and C.~Leng, ``Provable subspace clustering: When lrr meets
  ssc,'' \emph{Advances in Neural Information Processing Systems (NIPS)}, 2013.

\bibitem{Elhamifar:TPAMI14}
E.~Elhamifar, G.~Sapiro, and S.~S. Sastry, ``Dissimilarity-based sparse subset
  selection,'' \emph{{IEEE} Transactions on Pattern Analysis and Machine
  Intelligence}, 2016.

\bibitem{Donoho:CPAM06-2}
D.~L. Donoho, ``For most large underdetermined systems of linear equations, the
  minimal $ell^1$-norm near-solution approximates the sparsest near-solution,''
  \emph{Communications on Pure and Applied Mathematics}, vol.~59, no.~7, pp.
  907--934, 2006.

\bibitem{Candes:RIP08}
E.~J. Cand\`es, ``The restricted isometry property and its implications for
  compressed sensing,'' in \emph{Compte Rendus de l'Academie des Sciences,
  Paris, Serie I}, vol. 346, 2008, pp. 589--592.

\bibitem{Wang:ICML13}
Y.~Wang and H.~Xu, ``Noisy sparse subspace clustering,'' \emph{International
  Conference on Machine Learning (ICML)}, 2013.

\bibitem{Candes-Romberg-Tao:CPAM06}
E.~J. Cand\`es, J.~Romberg, and T.~Tao, ``Stable signal recovery from
  incomplete and inaccurate measurements,'' \emph{Communications on Pure and
  Applied Mathematics}, vol.~59, no.~8, pp. 1207--1223, 2006.

\bibitem{Tropp:TIT06}
J.~A. Tropp, ``Just relax: Convex programming methods for identifying sparse
  signals in noise,'' \emph{IEEE Transactions on Information Theory}, vol.~52,
  no.~3, 2006.

\bibitem{DonohoElad:PNAS03}
D.~L. Donoho and M.~Elad, ``Optimally sparse representation in general
  (nonorthogonal) dictionaries via $\ell_1$ minimization,'' \emph{PNAS}, vol.
  100, no.~5, pp. 2197--2202, 2003.

\bibitem{Gribonval:TIT03}
R.~Gribonval and M.~Nielsen, ``Sparse representations in unions of bases,''
  \emph{IEEE Trans. Information Theory}, vol.~49, no.~12, pp. 3320--3325, Dec.
  2003.

\bibitem{Stojnic:TSP09}
M.~Stojnic, F.~Parvaresh, and B.~Hassibi, ``On the reconstruction of
  block-sparse signals with and optimal number of measurements,'' \emph{IEEE
  Trans. Signal Processing}, vol.~57, no.~8, pp. 3075--3085, Aug. 2009.

\bibitem{VandenBerg:TIT10}
E.~van~den Berg and M.~Friedlander, ``Theoretical and empirical results for
  recovery from multiple measurements,'' \emph{IEEE Trans. Information Theory},
  vol.~56, no.~5, pp. 2516--2527, 2010.

\bibitem{Laurent:ASTAT00}
B.~Laurent and P.~Massart, ``Adaptive estimation of a quadratic functional by
  model selection,'' \emph{The Annals of Statistics}, vol.~28, no.~5, pp.
  1302--1338, 2000.

\end{thebibliography}
}

\end{document}